\pdfoutput=1

\documentclass[11pt]{article}

\usepackage[]{naacl2021}

\usepackage{times}
\usepackage{latexsym}

\usepackage[T1]{fontenc}

\usepackage[utf8]{inputenc}
\usepackage{fdsymbol}
\usepackage{stfloats}
\usepackage{microtype}
\usepackage{adjustbox}

\usepackage{amsmath}
\usepackage{url}
\usepackage{amsfonts}
\usepackage[ruled,linesnumbered]{algorithm2e}
\usepackage{booktabs} 
\usepackage{threeparttable}
\usepackage{graphicx}
\usepackage{placeins}
\usepackage{subcaption}

\usepackage{bbm}
\usepackage{amsthm}
\usepackage{multirow}

\newtheorem{theorem}{Theorem}

\usepackage{multirow}
\usepackage{longtable}

\usepackage{array}
\newcolumntype{L}[1]{>{\raggedright\let\newline\\\arraybackslash\hspace{0pt}}m{#1}}
\newcolumntype{C}[1]{>{\centering\let\newline\\\arraybackslash\hspace{0pt}}m{#1}}
\newcolumntype{R}[1]{>{\raggedleft\let\newline\\\arraybackslash\hspace{0pt}}m{#1}}
\newcommand{\mbs}{\boldsymbol}
\title{PCFGs Can Do Better: Inducing \\ Probabilistic Context-Free Grammars with Many Symbols}

\author{Songlin Yang$^{\clubsuit}$, Yanpeng Zhao$^{\diamondsuit}$, Kewei Tu$^{\clubsuit}$\thanks{\, 	 Corresponding Author}\\
  $^{\clubsuit}$School of Information Science and Technology, ShanghaiTech University \\
  \textsuperscript{}{Shanghai Engineering Research Center of Intelligent Vision and Imaging}\\
  \textsuperscript{}{Shanghai Institute of Microsystem and Information Technology, Chinese Academy of Sciences}\\
    \textsuperscript{}{University of Chinese Academy of Sciences}\\
 $^{\diamondsuit}${ILCC, University of Edinburgh}\\
    {\tt \{yangsl,tukw\}@shanghaitech.edu.cn}\\
    {\tt yannzhao.ed@gmail.com}
 }

\begin{document}
\maketitle
\begin{abstract}
Probabilistic context-free grammars (PCFGs) with neural parameterization have been shown to be effective in unsupervised phrase-structure grammar induction.
However, due to the cubic computational complexity of PCFG representation and parsing,
previous approaches cannot scale up to a relatively large number of (nonterminal and preterminal) symbols.
In this work, we present a new parameterization form of PCFGs based on tensor decomposition,
which has at most quadratic computational complexity in the symbol number and therefore allows us to use a much larger number of symbols.
We further use neural parameterization for the new form to improve unsupervised parsing performance.
We evaluate our model across ten languages and empirically demonstrate the effectiveness of using more symbols. \footnote{Our code: https://github.com/sustcsonglin/TN-PCFG}
\end{abstract}

\section{Introduction}

Unsupervised constituency parsing is the task of inducing phrase-structure grammars from raw text without using parse tree annotations.
Early work induces probabilistic context-free grammars (PCFGs) via the Expectation Maximation algorithm and finds the result unsatisfactory~\citep{lari1990estimation,carroll1992two}.
Recently, PCFGs with neural parameterization (i.e., using neural networks to generate rule probabilities) have been shown to achieve good results in unsupervised constituency parsing \cite{kim-etal-2019-compound,jin-etal-2019-unsupervised,zhu2020return}.
However, due to the cubic computational complexity of PCFG representation and parsing, these approaches learn PCFGs with relatively small numbers of nonterminals and preterminals. For example, \citet{jin-etal-2019-unsupervised} use 30 nonterminals (with no distinction between preterminals and other nonterminals) and \citet{kim-etal-2019-compound} use 30 nonterminals and 60 preterminals.

In this paper, we study PCFG induction with a much larger number of nonterminal and preterminal symbols.
We are partly motivated by the classic work of latent variable grammars in supervised constituency parsing~\citep{matsuzaki-etal-2005-probabilistic,petrov-etal-2006-learning,liang-etal-2007-infinite,cohen-etal-2012-spectral,zhao-etal-2018-gaussian}. While the Penn treebank grammar contains only tens of nonterminals and preterminals, it has been found that dividing them into subtypes could significantly improves the parsing accuracy of the grammar. For example, the best model from \citet{petrov-etal-2006-learning} contains over 1000 nonterminal and preterminal symbols.
We are also motivated by the recent work of \citet{buhai2019empirical} who show that when learning latent variable models, increasing the number of hidden states is often helpful; and by \citet{chiu-rush-2020-scaling} who show that a neural hidden Markov model with up to $2^{16}$ hidden states can achieve surprisingly good performance in language modeling. 

A major challenge in employing a large number of nonterminal and preterminal symbols is that representing and parsing with a PCFG requires a computational complexity that is cubic in its symbol number.
To resolve the issue, we rely on a new parameterization form of PCFGs based on tensor decomposition, which reduces the computational complexity from cubic to at most quadratic.
Furthermore, we apply neural parameterization to the new form, which is crucial for boosting unsupervised parsing performance of PCFGs as shown by \citet{kim-etal-2019-compound}.

We empirically evaluate our approach across ten languages.
On English WSJ, 
our best model with 500 preterminals and 250 nonterminals 
improves over the model with 60 preterminals and 30 nonterminals by 6.3\% mean F1 score, and we also observe consistent decrease in perplexity and overall increase in F1 score with more symbols in our model,
thus confirming the effectiveness of using more symbols.
Our best model also surpasses the strong baseline Compound PCFGs~\citep{kim-etal-2019-compound} by 1.4\% mean F1.
We further conduct multilingual evaluation on nine additional languages.
The evaluation results suggest good generalizability of our approach on languages beyond English.

Our key contributions can be summarized as follows:
(1) We propose a new parameterization form of PCFGs based on tensor decomposition, which enables us to use a large number of symbols in PCFGs.
(2) We further apply neural parameterization to improve unsupervised parsing performance.
(3) We evaluate our model across ten languages and empirically show the effectiveness of our approach.

\section{Related work}\label{sec:related}
\paragraph{Grammar induction using neural networks:} There is a recent resurgence of interest in unsupervised constituency parsing, mostly driven by neural network based methods  \cite{shen2018neural,shen2018ordered,drozdov-etal-2019-unsupervised, drozdov-etal-2020-unsupervised, kim-etal-2019-compound, kim-etal-2019-unsupervised,jin-etal-2019-unsupervised,zhu2020return}. These methods can be categorized into two major groups: those built on top of a generative grammar and those without a grammar component. The approaches most related to ours belong to the first category, which use neural networks to produce grammar rule probabilities. \citet{jin-etal-2019-unsupervised} use an invertible neural projection network (a.k.a. normalizing flow \cite{DBLP:conf/icml/RezendeM15}) to parameterize the preterminal rules of a PCFG. 
\citet{kim-etal-2019-compound} use neural networks to parameterize all the PCFG rules. \citet{zhu2020return} extend their work to lexicalized PCFGs, which are more expressive than PCFGs and can model both dependency and constituency parse trees simultaneously.

In other unsupervised syntactic induction tasks, there is also a trend to use neural networks to produce grammar rule probabilities. In unsupervised dependency parsing, the Dependency Model with Valence (DMV) \cite{klein-manning-2004-corpus} has been parameterized neurally to achieve higher induction accuracy \cite{jiang-etal-2016-unsupervised, songlin2020second}. In part-of-speech (POS) induction, neurally parameterized Hidden Markov Models (HMM) also achieve state-of-the-art results \cite{tran-etal-2016-unsupervised, he-etal-2018-unsupervised}.

\paragraph{Tensor decomposition on PCFGs:} Our work is closely related to \citet{cohen-etal-2013-approximate} in that both use tensor decomposition to parameterize the probabilities of binary rules for the purpose of reducing the time complexity of the inside algorithm. However, \citet{cohen-etal-2013-approximate} use this technique to speed up inference of an existing PCFG, and they need to actually perform tensor decomposition on the rule probability tensor of the PCFG.
In contrast, we draw inspiration from this technique to design a new parameterization form of PCFG that can be directly learned from data. Since we do not have a probability tensor to start with, additional tricks have to be inserted in order to ensure validity of the parameterization, as will be discussed later.




\section{Background}

\subsection{Tensor form of PCFGs}\label{sec:t-pcfg}
PCFGs build upon context-free grammars (CFGs). 
We start by introducing CFGs and establishing notations.
A CFG is defined as a 5-tuple $\mathcal{G} = (\mathcal{S},\mathcal{N}, \mathcal{P},\Sigma, \mathcal{R})$ where $\mathcal{S}$ is the start symbol, $\mathcal{N}$ is a finite set of nonterminal symbols, $\mathcal{P}$ is a finite set of preterminal symbols,\footnote{Strictly, CFGs do not distinguish nonterminals $\mathcal{N}$ (constituent labels) from preterminals $\mathcal{P}$ (part-of-speech tags). 
	They are both treated as nonterminals. 
	$\mathcal{N}, \mathcal{P}, \Sigma$ satisfy $\mathcal{N}\cap\mathcal{P}=\emptyset$ and $(\mathcal{N}\cup\mathcal{P})\cap\Sigma=\emptyset$.} 
$\Sigma$ is a finite set of terminal symbols, and $\mathcal{R}$ is a set of rules in the following form:
\begin{align*}
&S \rightarrow A   & \quad A \in \mathcal{N} \\
&A \rightarrow B C, &\quad A \in \mathcal{N}, \quad B, C \in \mathcal{N} \cup \mathcal{P} \\ &T\rightarrow w, & \quad T \in \mathcal{P}, w \in \Sigma  
\end{align*}
PCFGs extend CFGs by associating each rule $r\in\mathcal{R}$ with a probability $\pi_{r}$.
Denote $n$, $p$, and $q$ 
as the number of symbols in $\mathcal{N}$, $\mathcal{P}$, and $\Sigma$, respectively. 
It is convenient to represent the probabilities of the binary rules in the tensor form:
\begin{align*}
\mathbf{T}_{h_A, h_B, h_C} = \pi_{A\rightarrow BC}\,,\quad\mathbf{T}\in\mathbb{R}^{n\times m\times m}\,,
\end{align*}
where $\mathbf{T}$ is an order-3 tensor, $m = n + p$, and
$h_A\in[0, n)$ and $h_B$, $h_C\in[0, m)$ are symbol indices.
For the convenience of computation, 
we assign indices $[0, n)$ to nonterminals in $\mathcal{N}$ and $[n, m)$ to preterminals in $\mathcal{P}$.
Similarly, for a preterminal rule we define
\begin{align*}
\mathbf{Q}_{h_T, h_w} = \pi_{T\rightarrow w}\,,\quad\mathbf{Q}\in\mathbb{R}^{p\times q}\,.
\end{align*}
Again, $h_T$ and $h_w$ are the preterminal index and the terminal index, respectively. Finally, for a start rule we define
\begin{align*}
\mathbf{r}_{h_A} = \pi_{S\rightarrow A}\,,\quad\mathbf{r}\in\mathbb{R}^{n}\,.
\end{align*}

Generative learning of PCFGs involves maximizing the log-likelihood of every observed sentence $\mbs{w}=w_1,\ldots,w_l$:
\begin{align*}
\log p_{\theta}(\mbs{w}) = \log\sum_{t\in T_{\mathcal{G}}(\mbs{w})}  p(t)\,,
\end{align*}
where $T_{\mathcal{G}}(\mbs{w})$ contains all the parse trees of the sentence $\mbs{w}$ under a PCFG $\mathcal{G}$. The probability of a parse tree $t \in T_{\mathcal{G}}$ is defined as $p(t) = \prod_{r\in t_{\mathcal{R}}} \pi_{r}$, where $t_{\mathcal{R}}$ is the set of rules used in the derivation of t. $\log p_{\theta}(\mbs{w})$ can be estimated efficiently through the inside algorithm, which is fully differentiable and amenable to gradient optimization methods. 


	
	
	
	
		

\subsection{Tensor form of the inside algorithm}\label{sec:t-inside}

We first pad $\mathbf{T}$, $\mathbf{Q}$, and $\mathbf{r}$ with zeros such that $\mathbf{T}\in\mathbb{R}^{m\times m\times m}$,  $\mathbf{Q}\in\mathbb{R}^{m\times q}$, $\mathbf{r}\in\mathbb{R}^{m}$,
and all of them can be indexed by both nonterminals and preterminals.

The inside algorithm computes the probability of a symbol $A$ spanning a substring $\mbs{w}_{i, j} = w_{i},\ldots,w_j$ in a recursive manner ($0\leq i < j < l$):
\begin{align}\label{eq:inside_score}
s_{i, j}^{A} = \sum_{k = i}^{j - 1} \sum_{B, C}\pi_{A\rightarrow BC} \cdot 
s_{i, k}^{B} \cdot s_{k + 1, j}^{C}\,. \\
\text{Base Case:}\quad s_{i, i}^{T} = \pi_{T\rightarrow w_i}\,,  0\leq i < l\,. \nonumber
\end{align}
We use the tensor form of PCFGs to rewrite Equation~\ref{eq:inside_score} as:
\begin{align}
\mathbf{s}_{i, j}^{h_A} &= \sum_{k = i}^{j - 1}  \sum_{h_B, h_C}
\mathbf{T}_{h_A, h_B, h_C} \cdot \mathbf{s}_{i, k}^{h_B} \cdot \mathbf{s}_{k + 1, j}^{h_C}\nonumber \\
&= \sum_{k = i}^{j - 1}\left(\mathbf{T}_{h_A}\cdot \mathbf{s}_{k + 1, j}\right)\cdot \mathbf{s}_{i, k}\,,
\end{align}
where $\mathbf{s}_{i, j}$, $\mathbf{s}_{i, k}$, and $\mathbf{s}_{k+1, j}$ are all $m$-dimensional vectors; 
the dimension $h_A$ corresponds to the symbol $A$. 
Thus
\begin{align}\label{eq:inside_tensor}
\mathbf{s}_{i, j} = \sum_{k = i}^{j - 1}\left(\mathbf{T}\cdot \mathbf{s}_{k + 1, j}\right)\cdot \mathbf{s}_{i, k}\,.
\end{align}

Equation~\ref{eq:inside_tensor} represents the core computation of the inside algorithm as tensor-vector dot product.
It is amenable to be accelerated on a parallel computing device such as GPUs.
However, the time and space complexity is cubic in $m$,
which makes it impractical to use a large number of nonterminals and preterminals.

\section{Parameterizing PCFGs based on tensor decomposition}\label{sec:td-pcfg}

The tensor form of the inside algorithm has a high computational complexity of $\mathcal{O}(m^3l^3)$.
It hinders the algorithm from scaling to a large $m$.
To resolve the issue, we resort to a new parameterization form of PCFGs based on tensor decomposition (TD-PCFGs)~\citep{cohen-etal-2013-approximate}.
As discussed in Section \ref{sec:related}, while \citet{cohen-etal-2013-approximate} use a TD-PCFG to approximate an existing PCFG for speedup in parsing,
we regard a TD-PCFG as a stand-alone model and learn it directly from data.

The basic idea behind TD-PCFGs is using Kruskal decomposition of the order-3 tensor $\mathbf{T}$.
Specifically, we require $\mathbf{T}$ to be in the Kruskal form, 
\begin{align}\label{eq:kruskal}
\mathbf{T} = \sum_{l = 1}^{d} \mathbf{T}^{(l)}\,,\,\,\, 
\mathbf{T}^{(l)} = \mathbf{u}^{(l)}\otimes\mathbf{v}^{(l)}\otimes\mathbf{w}^{(l)}\,,
\end{align}
where $\mathbf{u}^{(l)}\in\mathbb{R}^{n}$ is a column vector of a matrix $\mathbf{U}\in\mathbb{R}^{n\times d}$; $\mathbf{v}^{(l)}$, $\mathbf{w}^{(l)}\in\mathbb{R}^{m}$ are column vectors of matrices $\mathbf{V}$, $\mathbf{W}\in\mathbb{R}^{m\times d}$, respectively;
$\otimes$ indicates Kronecker product.
Thus $\mathbf{T}^{(l)}\in\mathbb{R}^{n\times m\times m}$ is an order-3 tensor 
and 
$$\mathbf{T}^{(l)}_{i, j, k} = \mathbf{u}^{(l)}_{i}\cdot\mathbf{v}^{(l)}_{j}\cdot\mathbf{w}^{(l)}_{k}\,.$$

The Kruskal form of the tensor $\mathbf{T}$ is crucial for reducing the computation of Equation~\ref{eq:inside_tensor}.
To show this, we let $\mathbf{x} = \mathbf{s}_{i, k}$, $\mathbf{y} = \mathbf{s}_{k + 1, j}$, and $\mathbf{z}$ be any summand in the right-hand side of Equation~\ref{eq:inside_tensor}, so we have:
\begin{align}\label{eq:inside_general}
\mathbf{z} = \left(\mathbf{T}\cdot\mathbf{y}\right)\cdot\mathbf{x}\,.
\end{align}
Substitute $\mathbf{T}$ in Equation~\ref{eq:kruskal} into Equation~\ref{eq:inside_general} and consider the $i$-th dimension of $\mathbf{z}$:
\begin{align}
\mathbf{z}_{i} &= (\mathbf{T}_{i}\cdot\mathbf{y})\cdot\mathbf{x} \nonumber \\
&= \sum_{j = 1}^{m}\sum_{k=1}^{m}\sum_{l = 1}^{d}\mathbf{T}_{i, j, k}^{(l)} \cdot\mathbf{x}_{j}\cdot\mathbf{y}_{k} \nonumber \\
&= \sum_{j = 1}^{m}\sum_{k=1}^{m}\sum_{l = 1}^{d} \mathbf{u}^{(l)}_{i}\cdot\mathbf{v}^{(l)}_{j}\cdot\mathbf{w}^{(l)}_{k} \cdot\mathbf{x}_{j}\cdot\mathbf{y}_{k} \nonumber \\
&= \sum_{l = 1}^{d} \mathbf{u}^{(l)}_{i}\cdot
\left(\sum_{j = 1}^{m}\mathbf{v}^{(l)}_{j}\cdot\mathbf{x}_{j}\right)\cdot
\left(\vphantom{\sum_{j = 1}^{m}}\sum_{k=1}^{m}\mathbf{w}^{(l)}_{k} \cdot\mathbf{y}_{k}\right) \nonumber \\
&= \sum_{l = 1}^{d} \mathbf{u}^{(l)}_{i}\cdot
\left(\mathbf{x}^{T}{\mathbf{v}^{(l)}}\right)\cdot
\left(\mathbf{y}^{T}{\mathbf{w}^{(l)}}\right) \nonumber\\
\label{eq:kruskal_hadamard}
&= \left(\mathbf{e}_{i}^{T}\mathbf{U}\right)\cdot
\left(\left(\mathbf{V}^{T}{\mathbf{x}}\right)\odot
\left(\mathbf{W}^{T}{\mathbf{y}}\right)\right)\,, 
\end{align}
where $\odot$ indicates Hadamard (element-wise) product;
$\mathbf{e}_i\in\mathbb{R}^{m}$ is a one-hot vector that selects the $i$-th row of $\mathbf{U}$. 
We have padded $\mathbf{U}$ with zeros such that $\mathbf{U}\in\mathbb{R}^{m\times d}$ and the last $m-n$ rows are all zeros.
Thus
\begin{align}
\mathbf{z} = \mathbf{U}\cdot\left(\left(\mathbf{V}^{T}{\mathbf{x}}\right)\odot
\left(\mathbf{W}^{T}{\mathbf{y}}\right)\right)\,,
\end{align}
and accordingly,
\begin{align}\label{eq:inside_kruskal}
\mathbf{s}_{i, j} = \mathbf{U}\cdot\sum_{k = i}^{j - 1}
\left(\left(\mathbf{V}^{T}{\mathbf{s}_{i, k}}\right)\odot \left(\mathbf{W}^{T}{\mathbf{s}_{k + 1, j}}\right)\right)\,.
\end{align}
Equation~\ref{eq:inside_kruskal} computes the inside probabilities using TD-PCFGs.
It has a time complexity $\mathcal{O}(md)$.
By caching $\mathbf{V}^{T}\mathbf{s}_{i,k}$ and $\mathbf{W}^{T}\mathbf{s}_{k+1, j}$, the time complexity of the inside algorithm becomes $\mathcal{O}(dl^3+mdl^2)$ \cite{cohen-etal-2013-approximate}, which is at most quadratic in $m$ since we typically set $d = \mathcal{O}(m)$. Interestingly, Equation~\ref{eq:inside_kruskal} has similar forms to recursive neural networks \cite{socher-etal-2013-recursive} if we treat inside score vectors as span embeddings.

One problem with TD-PCFGs is that, since we use three matrices $\mathbf{U}, \mathbf{V}$ and $\mathbf{W}$ to represent tensor $\mathbf{T}$ of binary rule probabilities, how we can ensure that $\mathbf{T}$ is non-negative and properly normalized, i.e., for a given left-hand side symbol $A$, $\sum_{j, k}\mathbf{T}_{h_A, j, k} = 1$. 
Simply reconstructing $\mathbf{T}$ with $\mathbf{U}, \mathbf{V}$ and $\mathbf{W}$ and then performing normalization would take $\mathcal{O}(m^3)$ time, thus defeating the purpose of TD-PCFGs.
Our solution is to require that the three matrices are non-negative and meanwhile $\mathbf{U}$ is row-normalized and $\mathbf{V}$ and $\mathbf{W}$ are column-normalized~\citep{shen2018efficient}.

\begin{theorem}\label{prop:kruskal_form}
	Given non-negative matrices $\mathbf{U}\in\mathbb{R}^{n\times d} $ and $\mathbf{V},  \mathbf{W}\in\mathbb{R}^{m\times d}$,
	if $\mathbf{U}$ is row-normalized and $\mathbf{V}$ and $\mathbf{W}$ are column-normalized, then $\mathbf{U}$, $\mathbf{V}$, and $\mathbf{W}$ are a Kruskal decomposition of a tensor $\mathbf{T}\in\mathbb{R}^{n\times m\times m}$ where $\mathbf{T}_{i, j, k}\in[0, 1]$ and $\mathbf{T}_{i}$ is normalized such that $\sum_{j, k}\mathbf{T}_{i, j, k} = 1$.
\end{theorem}
\begin{proof}
	\begin{align*}
	\sum_{j=1}^{m} \sum_{k=1}^{m} \mathbf{T}_{i, j, k} 
	&= \sum_{j=1}^{m} \sum_{k=1}^{m}\sum_{l=1}^{d} \mathbf{u}^{(l)}_{i}\cdot\mathbf{v}^{(l)}_{j}\cdot\mathbf{w}^{(l)}_{k} \\
	&= \sum_{l=1}^{d} \mathbf{u}^{(l)}_{i} \cdot (\sum_{j=1}^{m} \mathbf{v}^{(l)}_{j}) \cdot (\sum_{k=1}^{m} \mathbf{w}^{(l)}_{j}) \\
	&= \sum_{l=1}^{d} \mathbf{u}^{(l)}_{i} \cdot 1\cdot 1 = 1
	\end{align*}
\end{proof}

\section{Neural parameterization of TD-PCFGs}\label{sec:tn-pcfg}

We use neural parameterization for TD-PCFGs as it has demonstrated its effectiveness in inducing PCFGs~\citep{kim-etal-2019-compound}.
In a neurally parameterized TD-PCFGs,
the original TD-PCFG parameters are generated by neural networks,
rather than being learned directly;
parameters of the neural network will thus be the parameters to be optimized.
This modeling approach breaks the parameter number limit of the original TD-PCFG, so we can control the total number of parameters flexibly. When the total number of symbols is small, we can over-parameterize the model as over-parameterization has been shown to ease optimization \cite{DBLP:conf/icml/AroraCH18, DBLP:conf/nips/XuHM18, DBLP:conf/iclr/DuZPS19}. On the other hand, when the total number of symbols is huge, we can decrease the number of parameters to save GPU memories and speed up training. 

Specifically, we use neural networks to generate the following set of parameters of a TD-PCFG:
\begin{align*}
\Theta = \{
\mathbf{U}\,,\,\, \mathbf{V}\,,\,\, \mathbf{W}\,,\,\,
\mathbf{Q}\,,\,\, \mathbf{r}
\}\,.
\end{align*}
The resulting model is referred to as neural PCFGs based on tensor decomposition (TN-PCFGs).

We start with the neural parameterization of $\mathbf{U}\in\mathbb{R}^{n\times d}$ and $\mathbf{V}, \mathbf{W}\in\mathbb{R}^{m\times d}$.
We use shared symbol embeddings $\mathbf{E}_s\in\mathbb{R}^{m\times k}$ ($k$ is the symbol embedding dimension) in which each row is the embedding of a nonterminal or preterminal.
We first compute an unnormalized $\tilde{\mathbf{U}}$ by applying a neural network $f_{u}(\cdot)$ to symbol embeddings $\mathbf{E}_s$:
\begin{align*}
\tilde{\mathbf{U}} = f_u(\mathbf{E}_s)= \left(\text{ReLU}\left(\mathbf{E}_s \mathbf{M}_{u}^{(1)}\right)\right) \mathbf{M}_{u}^{(2)}\,,
\end{align*}
where $\mathbf{M}_{u}^{(1)}\in\mathbb{R}^{k \times k}$ and $\mathbf{M}_{u}^{(2)}\in\mathbb{R}^{k \times d}$ are learnable parameters of $f_{u}(\cdot)$. For simplicity, we omit the learnable bias terms.
We compute unnormalized $\tilde{\mathbf{V}}$ and $\tilde{\mathbf{W}}$ in a similar way.
Note that only $\mathbf{E}_s$ is shared in computing the three unnormalized matrices.
Then we apply the Softmax activation function to each row of $\tilde{\mathbf{U}}$ and to each column of $\tilde{\mathbf{V}}$ and $\tilde{\mathbf{W}}$,
and obtain normalized $\mathbf{U}$, $\mathbf{V}$, and $\mathbf{W}$.

For preterminal-rule probabilities $\mathbf{Q}\in\mathbb{R}^{p\times q}$ and start-rule probabilities $\mathbf{r}\in\mathbb{R}^{n}$, 
we follow~\cite{kim-etal-2019-compound} and define them as:
\begin{align*}
\mathbf{Q}_{h_T, h_w} = \pi_{T\rightarrow w} &= \frac{\exp(\mathbf{u}_{w}^{T} f_t(\mathbf{w}_{T}))}
{\sum_{w'\in\Sigma}\exp(\mathbf{u}_{w'}^{T} f_t(\mathbf{w}_{T}))} \,, \\
\mathbf{r}_{h_A} = \pi_{S\rightarrow A} &= \frac{\exp(\mathbf{u}_{A}^{T} f_s(\mathbf{w}_{S}))}
{\sum_{A'\in\mathcal{N}}\exp(\mathbf{u}_{A'}^{T} f_s(\mathbf{w}_{S}))} \,, \\
\end{align*}
where $\mathbf{w}$ and $\mathbf{u}$ are symbol embeddings;
$f_s(\cdot)$ and $f_t(\cdot)$ are neural networks that encode the input into a vector (see details in~\citet{kim-etal-2019-compound}). Note that the symbol embeddings are not shared between preterminal rules and start rules.

\section{Parsing with TD-PCFGs}\label{sec:inference}

Parsing seeks the most probable parse $t^\star$ from all the parses ${T_{\mathcal{G}}(\mbs{w})}$ of a sentence $\mbs{w}$:
\begin{align}
t^\star = \arg\max_{t\in T_{\mathcal{G}}(\mbs{w})} p(t | \mbs{w})\,.
\end{align}
Typically, the CYK algorithm\footnote{
	The CYK algorithm is similar to the inside algorithm. The only difference is that it uses \textsc{Max} whenever the inside algorithm performs \textsc{Sum} over $k$ and $B, C$ (\textit{cf}. Equation~\ref{eq:inside_score}).
} 
can be directly used to solve this problem exactly:
it first computes the score of the most likely parse;
and then automatic differentiation is applied to recover the best tree structure $t^\star$~\citep{eisner-2016-inside,rush-2020-torch}.
This, however, relies on the original probability tensor $\mathbf{T}$ and is incompatible with our decomposed representation.\footnote{
	In Equation~\ref{eq:inside_kruskal} all symbols become entangled through $\mathbf{V}^{T}{\mathbf{s}_{i, k}}$ and $\mathbf{W}^{T}{\mathbf{s}_{k + 1, j}}$.
	We are unable to perform \textsc{Max} over $B, C$ as in the CYK algorithm.
}
If we reconstruct $\mathbf{T}$ from $\mathbf{U}$, $\mathbf{V}$, $\mathbf{W}$ and then perform CYK, then the resulting time and space complexity would degrade to $\mathcal{O}(m^3l^3)$ and become unaffordable when $m$ is large.
Therefore, we resort to Minimum Bayes-Risk (MBR) style decoding because we can compute the inside probabilities efficiently.

Our decoding method consists of two stages.
The first stage computes the conditional probability of a substring $\mbs{w}_{i, j}$ being a constituent in a given sentence $\mbs{w}$ (a.k.a. posteriors of spans being a constituent):
\begin{align*}
p(\mbs{w}_{i, j} | \mbs{w}) = \frac{1}{p(\mbs{w})} \sum_{t\in T_{\mathcal{G}}(\mbs{w})} p(t)\cdot\mathbbm{1}_{\{ \mbs{w}_{i, j} \in t  \}}\,.
\end{align*}
We can estimate the posteriors efficiently by using automatic differentiation after obtaining all the inside probabilities. This has the same time complexity as our improved inside algorithm, which is $\mathcal{O}(dl^3+mdl^{2})$.
The second stage uses the CYK algorithm to find the parse tree that has the highest expected number of constituents \citep{DBLP:conf/acl/SmithE06a}:
\begin{align}
t^\star = \arg\max_{t\in T_{\mathcal{G}}(\mbs{w})} \sum_{\mbs{w}_{i, j} \in t } p(\mbs{w}_{i, j} | \mbs{w})\,.
\end{align}
The time complexity of the second stage is $\mathcal{O}(l^3)$, so the overall time complexity of our decoding method is $\mathcal{O}(dl^3+mdl^2)$, which is much faster than $\mathcal{O}(m^3l^3)$ in general.

\section{Experimental setup}
\subsection{Datasets}
We evaluate TN-PCFGs across ten languages.
We use the Wall Street Journal (WSJ) corpus of the Penn Treebank~\citep{marcus-etal-1994-penn} for English,
the Penn Chinese Treebank 5.1 (CTB)~\citep{xue_xia_chiou_palmer_2005} for Chinese,
and the SPRML dataset~\citep{seddah-etal-2014-introducing} for the other eight morphology-rich languages.
We use a unified data preprocessing pipeline\footnote{ \href{https://github.com/zhaoyanpeng/xcfg}{https://github.com/zhaoyanpeng/xcfg}.} provided by~\citet{zhao2020xcfg}. 
The same pipeline has been used in several recent papers~\citep{shen2018neural,shen2018ordered,kim-etal-2019-compound,zhao-titov-2020-visually}.
Specifically, for every treebank, punctuation is removed from all data splits
and the top 10,000 frequent words in the training data are used as the vocabulary.

\subsection{Settings and hyperparameters}
For baseline models we use the best configurations reported by the authors.
For example, we use 30 nonterminals and 60 preterminals for N-PCFGs and C-PCFGs.
We implement TN-PCFGs and reimplement N-PCFGs and C-PCFGs using automatic differentiation~\citep{eisner-2016-inside} and we borrow the idea of \citet{ijcai2020-560} to batchify the inside algorithm.
Inspired by~\citet{kim-etal-2019-compound},
for TN-PCFGs we set $n / p$, the ratio of the nonterminal number to the preterminal number, to $1 / 2$.
For $\mathbf{U}\in\mathbb{R}^{n\times d}$ and $\mathbf{V}, \mathbf{W}\in\mathbb{R}^{m\times d}$ we set $d=p$ when there are more than 200 preterminals and $d=200$ otherwise. The symbol embedding dimension $k$ is set to 256.
We optimize TN-PCFGs using the Adam optimizer~\citep{kingma2014adam} with $\beta_1 = 0.75$, $\beta_2 = 0.999$, and learning rate 0.001 with batch size 4. We use the unit Gaussian distribution to initialize embedding parameters. We do not use the curriculum learning strategy that is used by \citet{kim-etal-2019-compound} when training TN-PCFGs.

\subsection{Evaluation}
Following~\citet{kim-etal-2019-compound}, we train a TN-PCFG for each treebank separately.
For each setting we run the TN-PCFG and the baselines four times with different random seeds and for ten epochs each time.
Early stopping is performed based on the perplexity on the development data.
The best model in each run is selected according to the perplexity on the development data.
We tune model hyperparameters only on the development data of WSJ and use the same model configurations on the other treebanks.\footnote{
\citet{shi-etal-2020-role} suggest not using the gold parses of the development data for hyperparameter tuning and model selection in unsupervised parsing.
Here we still use the gold parses of the WSJ development set for the English experiments in order to conduct fair comparison with previous work. No gold parse is used in the experiments of any other language.
}
We report average sentence-level F1 score\footnote{Following \citet{kim-etal-2019-compound}, we remove all trivial spans (single-word spans and sentence-level spans). Sentence-level means that we compute F1 for each sentence and then average over all sentences. } as well as their biased standard deviations.

\begin{table}[t!] 	 
	{\setlength{\tabcolsep}{.5em}
		\makebox[\linewidth]{\resizebox{\linewidth}{!}{%
				\begin{tabular}{lll}
					\toprule 
					\multirow{2}{*}{Model} & \multicolumn{2}{c} { WSJ } \\
					\cmidrule{2-3}
					& Mean & Max   \\
					\midrule 
					Left Branching &  & \phantom{0}8.7   \\
					Right Branching &  & 39.5   \\
					Random Trees & 18.1 & 18.2 \\
					\midrule 
					\multicolumn{3}{c}{Systems without pretrained word embeddings}                                   \\
					\midrule 
					PRPN$^{\dagger}$~\citep{shen2018neural} & 47.3 & 47.9 \\
					ON$^{\dagger}$~\citep{shen2018ordered} & 48.1 & 50.0 \\
					N-PCFG~\citep{kim-etal-2019-compound} & 50.8 & 52.6 \\
					C-PCFG~\citep{kim-etal-2019-compound} & 55.2 & 60.1 \\
					NL-PCFG~\cite{zhu2020return} & 55.3 & \\
					\midrule 
					N-PCFG$^\star$ & 50.9$_{\pm 2.3}$ & 54.6 \\
					N-PCFG$^\star$ w/ MBR& 52.3$_{\pm 2.3}$ & 55.8 \\
					C-PCFG$^\star$ & 55.4$_{\pm 2.2}$ & 59.0 \\
					C-PCFG$^\star$ w/ MBR& 56.3$_{\pm 2.1}$ & 60.0\\ 
					TN-PCFG $p=60$ (ours)\phantom{0} & 51.4$_{\pm 4.0}$ & 55.6\\
					TN-PCFG $p=500$ (ours) & \textbf{57.7}$_{\pm 4.2}$ & \textbf{61.4}  \\
					\midrule
				\multicolumn{3}{c}{Systems with pretrained word embeddings}                                   \\
			
					\midrule 
					DIORA~\citep{drozdov-etal-2019-unsupervised}   &  & 56.8   \\
					S-DIORA~\citep{drozdov-etal-2020-unsupervised} & 57.6 & 64.0 \\					
					CT~\citep{cao-etal-2020-unsupervised} & 62.8 & 65.9 \\					 
					\midrule 
					Oracle Trees & & {84.3} \\
					\bottomrule
	\end{tabular}}}}
	\caption{\label{tab:main_wsj_ctb} Unlabeled sentence-level F1 scores on the WSJ test data.  
		$^\dagger$ indicates numbers reported by~\citet{kim-etal-2019-compound}.
		$^\star$ indicates our reimplementations of N-PCFGs and C-PCFGs.
		$p$ denotes the preterminal number. }
		\vskip -.1in
\end{table}

\begin{table}[h!]\small  
	{\setlength{\tabcolsep}{0.4em}
		\makebox[\linewidth]{\resizebox{\linewidth}{!}{%
				\begin{tabular}{lcc}
					\toprule
					\multicolumn{1}{l}{\textbf{Model}} &
					\multicolumn{1}{l}{\textbf{Time (minutes)}} &
					\multicolumn{1}{l}{\textbf{Total Parameters (M)}} \\
					\midrule
					N-PCFG & 19 & 5.3\\
					C-PCFG &  20 & 15.4\\ 
				    TN-PCFG ($p=60$) & 13 & 3.6 \\ 
				    TN-PCFG ($p=500$) & 26 & 4.2\\
					\bottomrule
	\end{tabular}}}}
	\caption{\label{tab:time}Average running time per epoch and the parameter number of each model.}
\end{table}

\section{Experimental results}

We evaluate our models mainly on WSJ (Section~\ref{sec:wsj}-\ref{sec:analysis}).
We first give an overview of model performance in Section~\ref{sec:wsj}
and then conduct ablation study of TN-PCFGs in Section~\ref{sec:ablation}.
We quantitatively and qualitatively analyze constituent labels induced by TN-PCFGs in Section~\ref{sec:analysis}.
In Section~\ref{sec:multilingual}, we conduct a multilingual evaluation over nine additional languages.

\subsection{Main results}\label{sec:wsj}
Our best TN-PCFG model uses 500 preterminals ($p=500$).
We compare it with a wide range of recent unsupervised parsing models (see the top section of Table~\ref{tab:main_wsj_ctb}).
Since we use MBR decoding for TN-PCFGs,
which produces higher F1-measure than the CYK decoding~\citep{goodman-1996-parsing},
for fair comparison we also use MBR decoding for our reimplemented N-PCFGs and C-PCFGs (see the middle section of Table~\ref{tab:main_wsj_ctb}).

We draw three key observations from Table~\ref{tab:main_wsj_ctb}:
(1) TN-PCFG ($p=500$) achieves the best mean and max F1 score. 
Notebly, it outperforms the strong baseline model C-PCFG by 1.4\% mean F1.
Compared with TN-PCFG ($p=60$),
TN-PCFG ($p=500$) brings a 6.3\% mean F1 improvement,
demonstrating the effectiveness of using more symbols.
(2) Our reimplementations of N-PCFGs and C-PCFGs are comparable to those of~\citet{kim-etal-2019-compound}, 
(3) MBR decoding indeed gives higher F1 scores (+1.4\% mean F1 for N-PCFG and +0.9\% mean F1 for C-PCFG).

In Table~\ref{tab:main_wsj_ctb} we also show the results of Constituent test (CT)~\citep{cao-etal-2020-unsupervised} and DIORA~\citep{drozdov-etal-2019-unsupervised,drozdov-etal-2020-unsupervised}, two recent state-of-the-art approaches. However, our work is not directly comparable to these approaches. CT relies on pretrained language models (RoBERTa)  and DIORA relies on pretrained word embeddings (context insensitive ELMo). 
In contrast, our model and the other approaches do not use pretrained word embeddings and instead learn word embeddings from scratch. 
We are also aware of URNNG \citep{kim-etal-2019-unsupervised}, which has a max F1 score of 45.4\%, but it uses punctuation and hence is not directly comparable to the models listed in the table.

We report the average running time\footnote{We measure the running time on a single Titan V GPU.} per epoch and the parameter numbers of different models in Table~\ref{tab:time}. 
We can see that TN-PCFG ($p=500$), which uses a much larger number of symbols, has even fewer parameters and is not significantly slower than N-PCFG.


\subsection{Influence of symbol number}\label{sec:ablation}
\begin{figure}[tb]
	\centering
	\includegraphics[width=1.\linewidth]{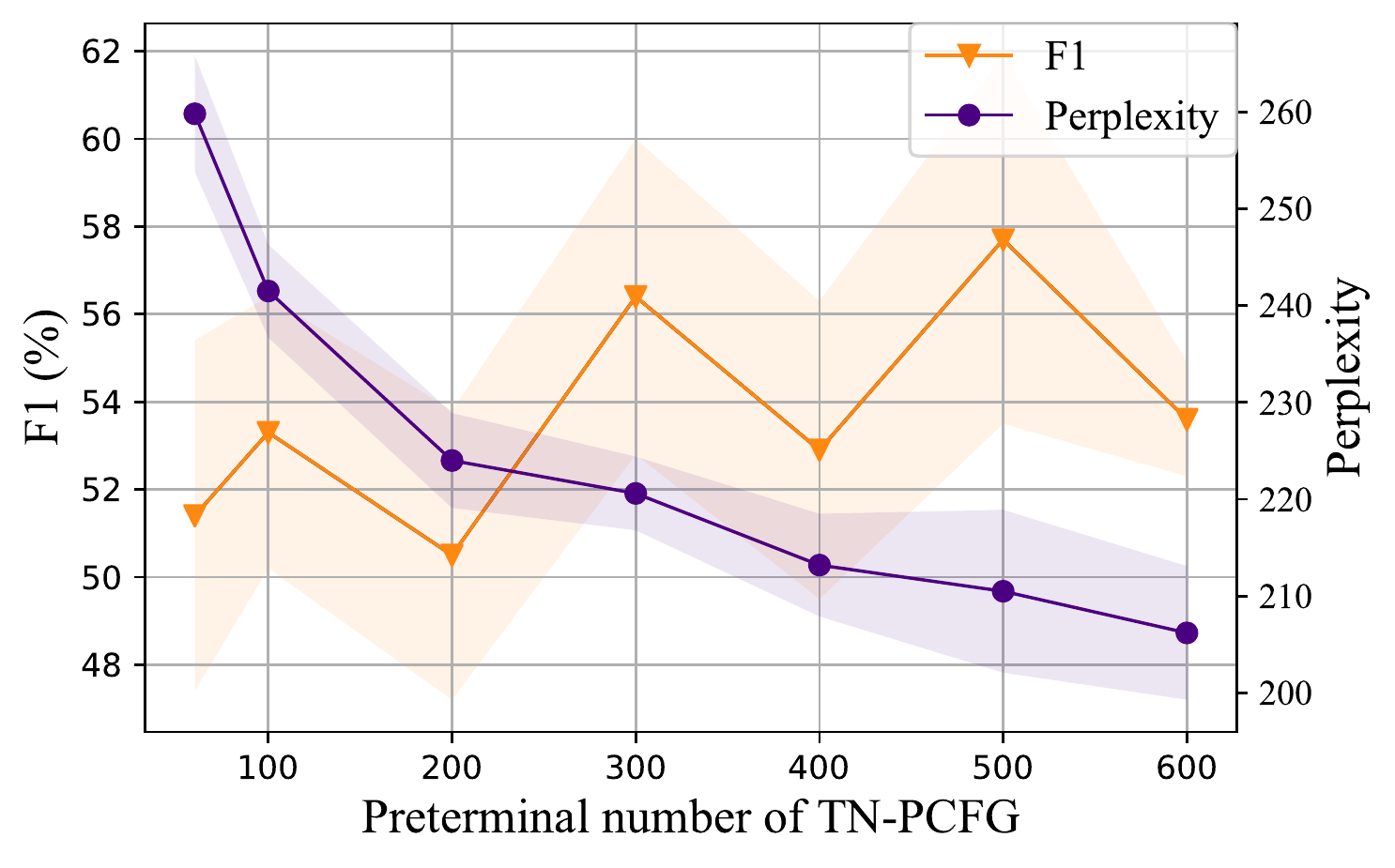}
	\caption{F1 scores and perplexities w.r.t. the preterminal number $p$ of TN-PCFGs on the WSJ test data. Recall that the nonterminal number $n$ is set to half of $p$.}
	\label{fig:f1_vs_perp}
	\vskip -.1in
\end{figure}

Figure~\ref{fig:f1_vs_perp} illustrates the change of F1 scores and perplexities as the number of nonterminals and preterminals increase.
We can see that, as the symbol number increases, the perplexities decrease while F1 scores tend to increase.

\subsection{Analysis on constituent labels}\label{sec:analysis}



We analyze model performance by breaking down recall numbers by constituent labels (Table~\ref{tab:recall_by_label_wsj}).
We use the top six frequent constituent labels in the WSJ test data (NP, VP, PP, SBAR, ADJP, and ADVP).
We first observe that the right-branching baseline remains competitive.
It achieves the highest recall on VPs and SBARs.
TN-PCFG ($p=500$) displays a relatively even performance across the six labels.
Specifically, it performs best on NPs and PPs among all the labels and it beats all the other models on ADJPs.
Compared with TN-PCFG ($p=60$),
TN-PCFG ($p=500$) results in the largest improvement on VPs (+19.5\% recall),
which are usually long (with an average length of 11) in comparison with the other types of constituents.
As NPs and VPs cover about 54\% of the total constituents in the WSJ test data,
it is not surprising that models which are accurate on these labels have high F1 scores (e.g., C-PCFGs and TN-PCFGs ($p=500$)). 
\begin{table*}[ht!]\small
	\centering
	{\setlength{\tabcolsep}{.8em}
		\makebox[\linewidth]{\resizebox{\linewidth}{!}{%
				\begin{tabular}{rllllllll}
					\toprule
					\multicolumn{1}{r}{\textbf{Model}} &
					\multicolumn{1}{l}{\textbf{NP}} & 
					\multicolumn{1}{l}{\textbf{VP}} & 
					\multicolumn{1}{l}{\textbf{PP}} & 
					\multicolumn{1}{l}{\textbf{SBAR}} & 
					\multicolumn{1}{l}{\textbf{ADJP}} & 
					\multicolumn{1}{l}{\textbf{ADVP}} &
					\multicolumn{1}{l}{\textbf{S-F1}}\\
					\midrule
					Left Branching & 10.4 & \phantom{0}0.5 & \phantom{0}5.0 & \phantom{0}5.3 & \phantom{0}2.5 & \phantom{0}8.0 & \phantom{0}8.7 \\
					Right Branching & 24.1 & \textbf{71.5} & 42.4 & \textbf{68.7} & 27.7 & 38.1 & 39.5 \\
					Random Trees & 22.5$_{\pm 0.3}$ & 12.3$_{\pm 0.3}$ & 19.0$_{\pm 0.5}$ & \phantom{0}9.3$_{\pm 0.6}$  & 24.3$_{\pm 1.7}$ & 26.9$_{\pm 1.3}$  &18.1$_{\pm 0.1}$\\
					NPCFG$^{\dagger}$ & 71.2 & 33.8 & 58.8 & 52.5 & 32.5 & 45.5 & 50.8 \\
					C-NPCFG$^{\dagger}$  & 74.7 &  41.7 & 68.8 &  56.1 & 40.4 & 52.5 & 55.2 \\
					\midrule
					N-PCFG$^\star$ w/ MBR & 72.3$_{\pm 3.6}$ &  28.1$_{\phantom{0}\pm6.6}$  & \textbf{73.0}$_{\phantom{0}\pm 2.6}$ & 53.6$_{\pm 10.0}$ & 40.8$_{\pm 4.2}$ & 43.8$_{\phantom{0} \pm 8.3}$ & 52.3$_{\pm 2.3}$ \\
					C-PCFG$^\star$ w/ MBR & 73.6$_{\pm 2.5}$ & 45.0$_{\phantom{0}\pm6.0}$ & 71.4$_{\phantom{0}\pm 1.4}$ & 54.8$_{\phantom{0}\pm 5.6}$ & 44.3$_{\pm 5.9}$ & \textbf{61.6}$_{\pm 17.6}$ & 56.3$_{\pm 2.1}$ \\
					TN-PCFG $p=60$\phantom{0} & \textbf{77.2}$_{\pm 2.6}$ & 28.9$_{\pm 13.8}$& 58.9$_{\pm 17.5}$ & 44.0$_{\pm 17.6}$ &  47.5$_{\pm 5.1}$ &  54.9$_{\phantom{0}\pm 6.0}$ & 51.4$_{\pm 4.0}$ \\
					TN-PCFG $p=500$ & 75.4$_{\pm 4.9}$ & 48.4$_{\pm 10.7}$ & 67.0$_{\pm 11.7}$ & 50.3$_{\pm 13.3}$ & \textbf{53.6}$_{\pm 3.3}$ & 59.5$_{\phantom{0}\pm 2.6}$ & \textbf{57.7}$_{\pm 4.2}$ \\
					\bottomrule
	\end{tabular}}}}
	\vskip -.1in
	\caption{
		Recall on the six frequent constituent labels in the WSJ test data,
		$^\dagger$ denotes results reported by~\citet{kim-etal-2019-compound},
		S-F1 represents sentence-level F1 measure.
	}
	\label{tab:recall_by_label_wsj}
	\vskip -.2in
\end{table*}


We further analyze the correspondence between the nonterminals of trained models and gold constituent labels.
For each model, we look at all the correctly-predicted constituents in the test set and estimate the empirical posterior distribution of nonterminals assigned to a constituent given the gold label of the constituent (see Figure~\ref{fig:align_heatmap}).
Compared with the other three models,
in TN-PCFG ($p=500$), the most frequent nonterminals are more likely to correspond to a single gold label.
One possible explanation is that it contains much more nonterminals and therefore constituents of different labels are less likely to compete for the same nonterminal.

\begin{figure*}[ht!]
	\begin{subfigure}[t]{0.24\linewidth}
		\includegraphics[scale=0.18]{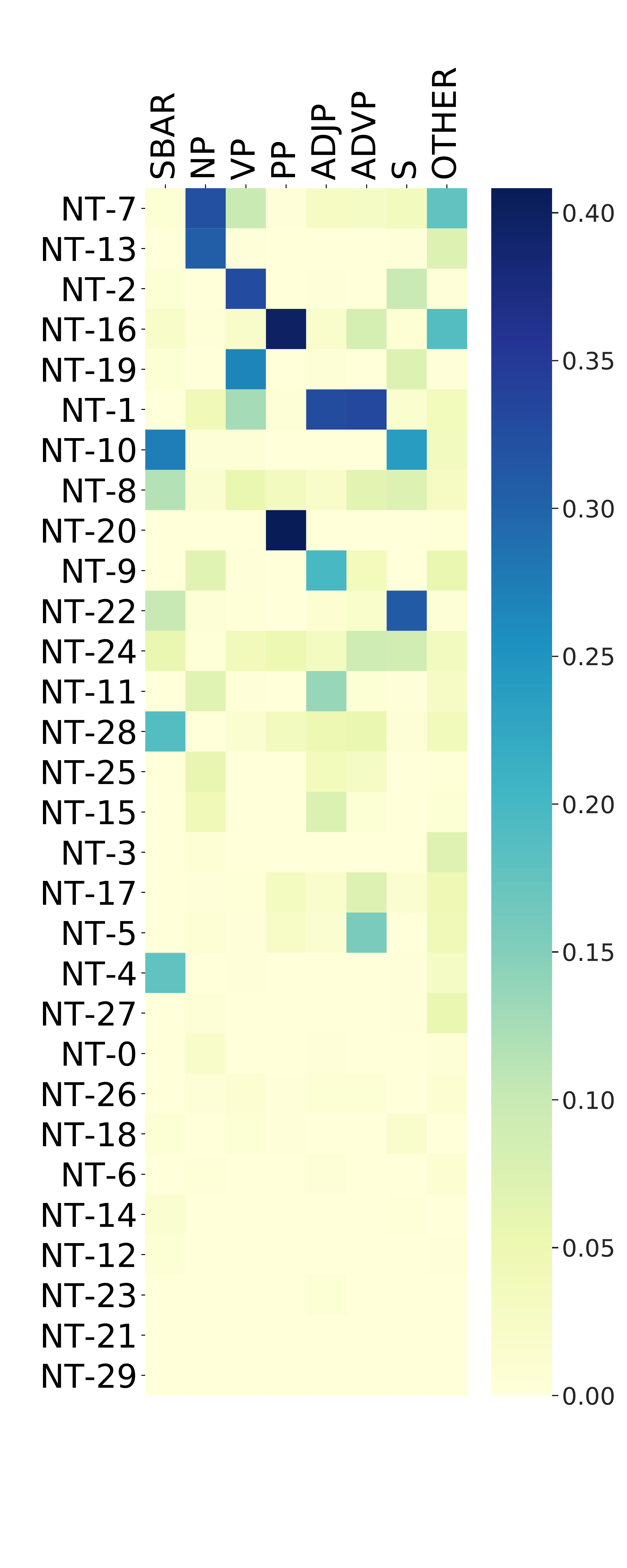}
		\caption{N-PCFG}
		\label{fig:label_induction_n-pcfg}
	\end{subfigure}
	\begin{subfigure}[t]{0.24\linewidth}
		\includegraphics[scale=0.18]{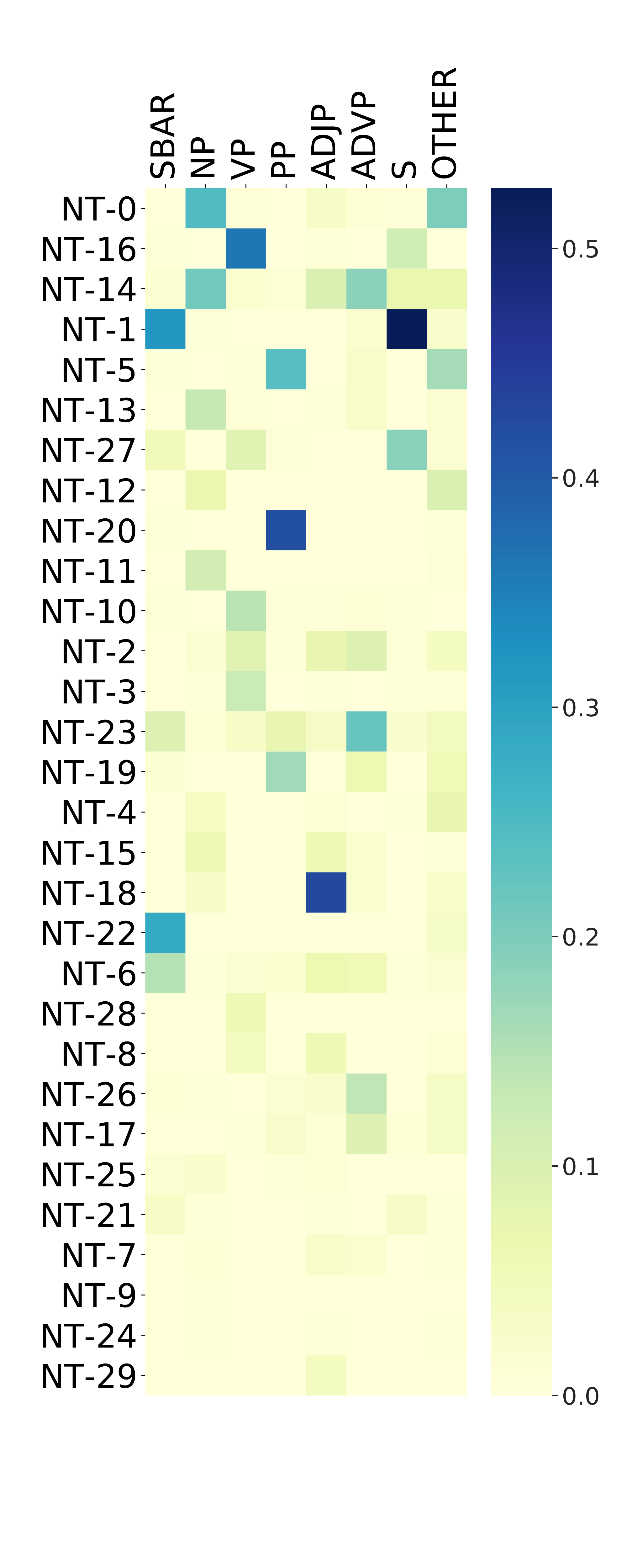}
		\caption{C-PCFG}
		\label{fig:label_induction_c-pcfg}
	\end{subfigure}
	\begin{subfigure}[t]{0.24\linewidth}
		\includegraphics[scale=0.18]{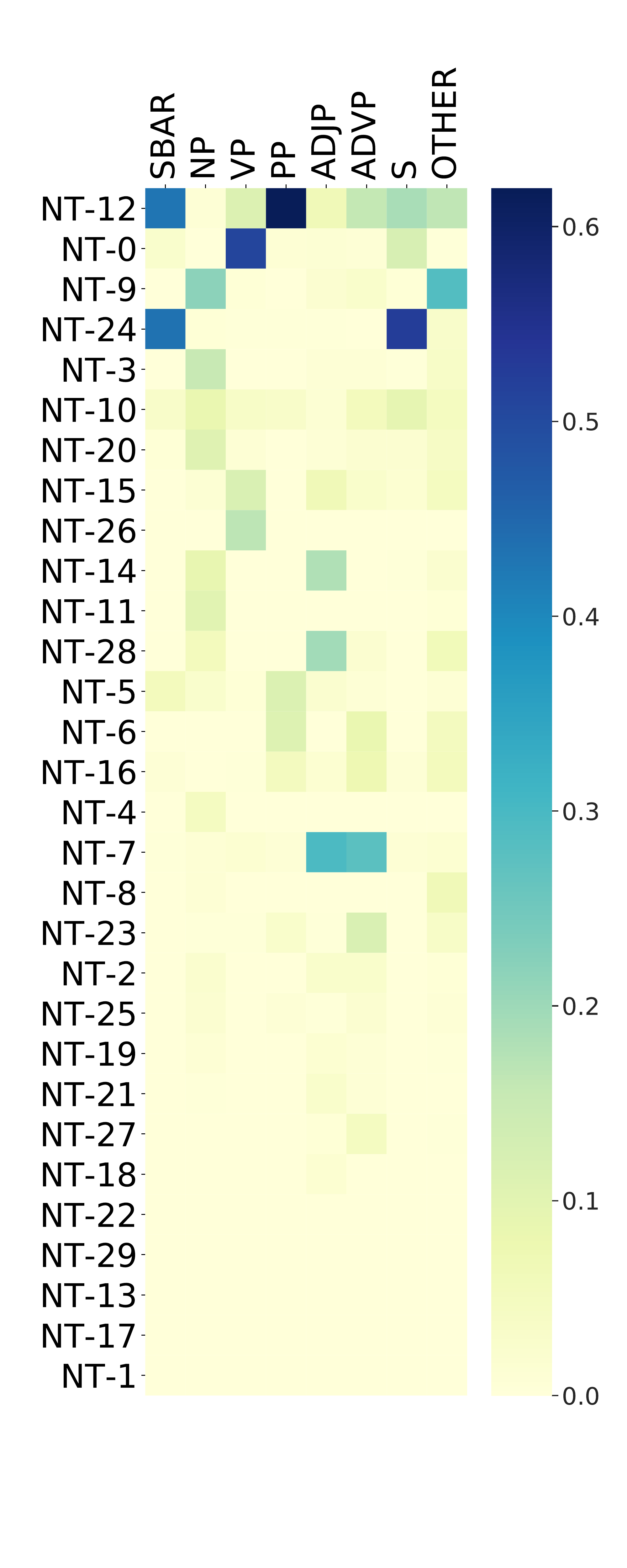}
		\caption{TN-PCFG ($p=60$)}
		\label{fig:label_induction_t60-pcfg}
	\end{subfigure}
	\begin{subfigure}[t]{0.24\linewidth}
		\includegraphics[scale=0.18]{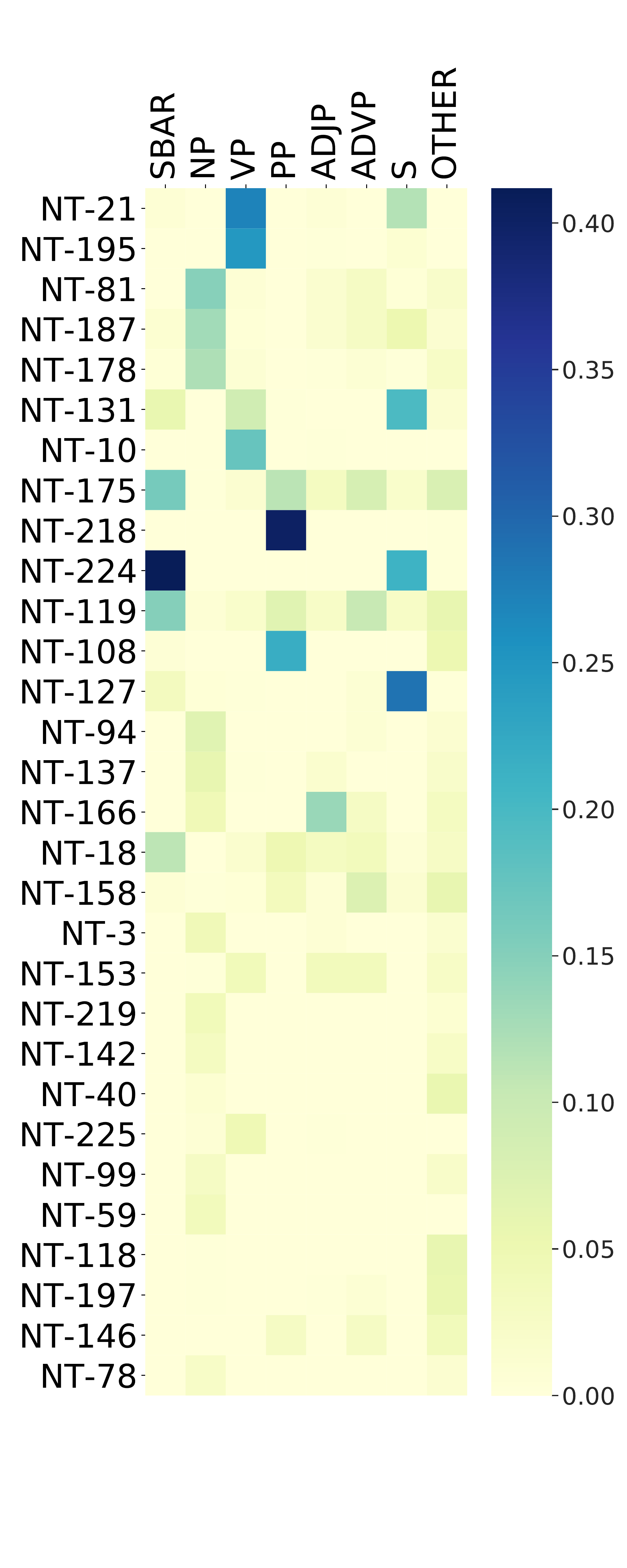}
		\caption{TN-PCFG ($p=500$)}
		\label{fig:label_induction_t500-pcfg}
	\end{subfigure}
	\caption{
	Correspondence between nonterminals and gold constituent labels. For each gold label, we visualize the proportion of correctly-predicted constituents that correspond to each nonterminal.
	Nonterminals NT-\# are listed in descending order of the prediction frequency (e.g., NT-21 is the most predicted nonterminal in TN-PCFG ($p=500$). Only the top 30 frequent nonterminals are listed.
	We show the seven most frequent gold labels and aggregate the rest (denoted by OTHER).
}
	\label{fig:align_heatmap}
	\vskip -.15in
\end{figure*}
Figure~\ref{fig:label_induction_t500-pcfg} (TN-PCFG ($p=500$)) also illustrates that a gold label may correspond to multiple nonterminals. 
A natural question that follows is: \textit{do these nonterminals capture different subtypes of the gold label?}
We find it is indeed the case for some nonterminals.
Take the gold label NPs (noun phrases),
while not all the nonterminals have clear interpretation, 
we find that NT-3 corresponds to constituents which represent a company name;
NT-99 corresponds to constituents which contain a possessive affix (e.g., ``'s'' in ``the market 's decline'');
NT-94 represents constituents preceded by an indefinite article.
We further look into the gold label PPs (preposition phrases).
Interestingly, NT-108, NT-175, and NT-218 roughly divided preposition phrases into three groups starting with `with, by, from, to', `in, on, for', and `of', respectively. See Appendix for more examples.

\subsection{Multilingual evaluation}\label{sec:multilingual}
\begin{table*}[htb!]\small
	\centering
	\vskip -.0in
	{\setlength{\tabcolsep}{.5em}
		\makebox[\linewidth]{\resizebox{\linewidth}{!}{%
				\begin{tabular}{rlllllllll|l}
					\toprule
					\textbf{Model} &
					\textbf{Chinese} &
					\textbf{Basque} &
					\textbf{German} &
					\textbf{French} &
					\textbf{Hebrew} &
					\textbf{Hungarian} &
					\textbf{Korean} &
					\textbf{Polish} &
					\textbf{Swedish} &
					\textbf{Mean} 
					\\ 
					\midrule
					Left Branching$^{\dagger}$ & \phantom{0}7.2 & 17.9 & 10.0 & \phantom{0}5.7 & \phantom{0}8.5 & 13.3 & 18.5 & 10.9 & \phantom{0}8.4 &  11.2 \\
					Right Branching$^{\dagger}$ & 25.5 & 15.4 & 14.7 & 26.4 & 30.0 & 12.7 & 19.2 & 34.2 & 30.4 & 23.2 \\
					Random Trees$^{\dagger}$ & 15.2 & 19.5 & 13.9 & 16.2 & 19.7 & 14.1 & 22.2 & 21.4 & 16.4 & 17.6 \\
					N-PCFG w/ MBR  & 26.3$_{\pm2.5}$  & 35.1$_{\pm 2.0}$ & 42.3$_{\pm 1.6}$ & \textbf{45.0}$_{\pm 2.0}$ & \textbf{45.7}$_{\phantom{0}\pm 2.2}$ & 43.5$_{\pm 1.2}$ & 28.4$_{\pm 6.5}$ & 43.2$_{\pm 0.8}$ & 17.0$_{\phantom{0}\pm 9.9}$ & 36.3 \\
					C-PCFG w/ MBR & 38.7$_{\pm 6.6}$ & \textbf{36.0}$_{\pm 1.2}$ & 43.5$_{\pm 1.2}$  & \textbf{45.0}$_{\pm 1.1}$ & 45.2$_{\phantom{0}\pm 0.5}$ & \textbf{44.9}$_{\pm 1.5}$ & 30.5$_{\pm 4.2}$ & 43.8$_{\pm 1.3}$ & 33.0$_{\pm 15.4}$ & 40.1 \\
					TN-PCFG $p=500$ &\textbf{39.2}$_{\pm5.0}$ & \textbf{36.0}$_{\pm 3.0}$ & \textbf{47.1}$_{\pm 1.7}$ & 39.1$_{\pm 4.1}$ & 39.2$_{\pm 10.7}$ & 43.1$_{\pm 1.1}$ & \textbf{35.4}$_{\pm 2.8}$ & \textbf{48.6}$_{\pm 3.1}$ & \textbf{40.0}$_{\phantom{0}\pm 4.8}$ & \textbf{40.9} \\
					\bottomrule				
	\end{tabular}}}}
	\caption{
		\label{tab:miltilingual} Sentence-level F1 scores on CTB and SPMRL. $^\dagger$ denotes results reported by~\citet{zhao2020xcfg}.
	}
	\vskip -.12in
\end{table*}

In order to understand the generalizability of TD-PCFGs on languages beyond English,
we conduct a multilingual evaluation of TD-PCFGs on CTB and SPMRL.
We use the best model configurations obtained on the English development data and do not perform any further tuning on CTB and SPMRL.
We compare TN-PCFGs with N-PCFGs and C-PCFGs and
use MBR decoding by default. The results are shown in Table~\ref{tab:miltilingual}.
In terms of the average F1 over the nine languages,
all the three models beat trivial left- and right-branching baselines by a large margin,
which suggests they have good generalizability on languages beyond English.
Among the three models, TN-PCFG ($p=500$) fares best.
It achieves the highest F1 score on six out of nine treebanks.
On Swedish, N-PCFG is worse than the right-branching baseline (-13.4\% F1),
while TN-PCFG ($p=500$) surpasses the right-branching baseline by 9.6\% F1.

\section{Discussions}
In our experiments, we do not find it beneficial to use the compound trick \cite{kim-etal-2019-compound} in TN-PCFGs, which is commonly used in previous work of PCFG induction \cite{kim-etal-2019-compound, zhao-titov-2020-visually, zhu2020return}. We speculate that the additional expressiveness brought by compound parameterization may not be necessary for a TN-PCFG with many symbols which is already sufficiently expressive; on the other hand, compound parameterization makes learning harder when we use more symbols. 

We also find neural parameterization and the choice of nonlinear activation functions greatly influence the performance. Without using neural parameterization, TD-PCFGs have only around $30\%$ S-F1 scores on WSJ, which are even worse than the right-branching baseline. Activation functions other than ReLU (such as tanh and sigmoid) result in much worse performance. It is an interesting open question why ReLU and neural parameterization are crucial in PCFG induction.   

When evaluating our model with a large number of symbols, we find that only a small fraction of the symbols are predicted in the parse trees (For example, when we use 250 nonterminals, only tens of nonterminals are used). We expect that our models can benefit from regularization techniques such as state dropout \cite{chiu-rush-2020-scaling}. 

\section{Conclusion}

We have presented TD-PCFGs, a new parameterization form of PCFGs based on tensor decomposition.
TD-PCFGs rely on Kruskal decomposition of the binary-rule probability tensor to reduce the computational complexity of PCFG representation and parsing from cubic to at most quadratic in the symbol number,
which allows us to scale up TD-PCFGs to a much larger number of (nonterminal and preterminal) symbols.
We further propose neurally parameterized TD-PCFGs (TN-PCFGs) and learn neural networks to produce the parameters of TD-PCFGs.
On WSJ test data, 
TN-PCFGs outperform strong baseline models;
we empirically show that using more nonterminal and preterminal symbols contributes to the high unsupervised parsing performance of TN-PCFGs.
Our multiligual evaluation on nine additional languages 
further reveals the capability of TN-PCFGs to generalize to languages beyond English.

\section*{Acknowledgements}
This work was supported by the National Natural Science Foundation of China (61976139).

\bibliography{anthology,custom}
\bibliographystyle{acl_natbib}

\appendix
\section{Appendix}

Table \ref{tab:analysis_nt_clusters} shows some representatives nonterminals and the corresponding constituent clusters. 

\begin{table*}[tbp]\small
	\centering
	\vskip -.0in
	\begin{tabular}{L{0.15\linewidth} | C{0.12\linewidth} | L{0.64\linewidth}}
					\toprule
					\textbf{Nonterminal} & \textbf{Description} & \multicolumn{1}{c}{\textbf{Constituents}}
					\\ 
					\midrule
	NT-3 & Company name & Quantum Chemical Corp. /  Dow Chemical Co. / Union Carbide Corp. / First Boston Corp. / 
	Petrolane Inc. / British Petroleum Co. / Ford Motor Co. / Jaguar PLC / Jaguar shares / General Motors Corp. / Northern Trust Co. / Norwest Corp. / Fidelity Investments / Windsor Fund 's / Vanguard Group Inc.
	\\
	\midrule
	NT-99 & Possesive affix & Quantum 's shares
	/ Quantum 's fortunes
	/	year 's end
	/	the company 's sway
	/	September 's steep rise
	/	the stock market 's plunge
	/	the market 's decline
	/	this year 's good inflation news
	/	the nation 's largest fund company
	/	Fidelity 's stock funds
	/	Friday 's close
	/	today 's closing prices
	/	this year 's fat profits \\
	\midrule
	NT-94 & Infinite article & an acquisition of Petrolane Inc.
	/ a transaction / a crisis	/ a bid	/	a minority interest	/	a role	/	a 0.7 \% gain	/	an imminent easing of monetary policy /	a blip /	a steep run-up/	a steeper rise /	a substantial rise	/	a broad-based advance	/	a one-time event /	a buffer /	a group /	an opportunity /	a buying opportunity \\
	\midrule
	NT-166 & Number & NP	99.8 million
	/	34.375 share /85 cents /a year ago /36 \% / 1.18 billion /15 \%
	/ eight members / 0.9 \% /  32.82 billion / 1.9 \%
	/ 1982 = 100 / about 2 billion / 8.8 \% and 9.2 \%  / Ten points / 167.7 million shares
	/ 20,000 shares /   about 2,000 stores / up to 500 a square foot
	/ 17 each / its 15 offering price / 190.58 points \\
	\midrule
	NT-59 / 142 / 219 & Determinater &	The timing	/ nearly a quarter	/	the bottom	/	the troughs	/	the chemical building block	/	only about half	/	the start	/	the day	/	the drawn-out process /	the possibility	/	the actions	/	the extent	/	the start	/	the pace	/	the beginning	/	the prices / less than 15 \% / two years	/	the volume	/	the chaos	/	the last hour \\
	\midrule
	NT-218 & of & of plastics / of the chemical industry / 	of the boom / of 99.8 million or 3.92 a share /	of its value / of Quantum 's competitors / of the heap / 	of its resources / of the plastics market / of polyethylene / its ethylene needs /of a few cents a pound / of intense speculation / 	of retail equity trading / of stock / of Securities Dealers  \\
	\midrule
	NT-108 & from, by, with, to & 	to Mr. Hardiman / to pre-crash levels / through the third quarter / from institutions / through the market / to Contel Cellular / with investors / to the junk bond market / with borrowed money / by the House / from the deficit-reduction bill / by the end of the month / to the streamlined bill / by the Senate action / with the streamlined bill / from the Senate bill / to top paid executives / like idiots \\
	\midrule
	NT-175 & for, on, in & for the ride / on plastics / for plastics producers / on plastics / for this article / at Oppenheimer \& Co. / on Wall Street / 	in a Sept. 29 meeting with analysts / in Morris Ill / in the hospital / between insured and insurer / 	for Quantum / in the U.S. / in the market 's decline /  top of a 0.7 \% gain in August / in the economy / in the producer price index / in the prices of consumer and industrial goods
	\\
	\midrule
	NT-158 & Start with PP & Within the next year / 	To make its point / summer /	Under the latest offer / At this price / In March /	In the same month / Since April / In the third quarter / Right now / As she calculates it / Unlike 1987 course / 	in both instances / in the October 1987 crash / 	In 1984 / Under the agreement /	In addition / 	In July / In computer publishing / On the other hand / In personal computers / At the same time / As a result of this trend / Beginning in mid-1987 / 	In addition \\
	\midrule
	NT-146 & PP in Sentence & 	in August / in September  /	in July / for September /	after Friday 's close /	in the House-passed bill / in the Senate bill / in fiscal 1990 / 	in tax increases / over five years / in air conditioners  / in Styrofoam / 	in federal fees / in fiscal 1990 / in Chicago / 	in Japan /	in Europe in Denver / at Salomon Bros / 	in California / in principle / in August / in operating costs / in Van Nuys / in San Diego / 	in Texas / 	in the dollar / in the past two years / 	in Dutch corporate law / 	since July / 	in 1964 / in June / for instance / in France / in Germany / 	in Palermo / for quotas \\
	\bottomrule		
	\end{tabular}
	\caption{
		\label{tab:analysis_nt_clusters} Nonterminals and the corresponding constituent clusters.
	}
	\vskip -.0in
\end{table*}

\label{sec:appendix}

\end{document}